\documentclass[letterpaper]{article}
\usepackage{aaai}
\usepackage{times}
\usepackage{helvet}
\usepackage{courier}

\usepackage{amsmath,amssymb, amsthm}
\usepackage{algorithm}
\usepackage{color}
\usepackage{graphicx}
\usepackage{pdfsync}
\usepackage{xfrac}

\usepackage{tikz}
\usepackage{pgfplots}
\pgfplotsset{compat=1.5} 
\usetikzlibrary{arrows,decorations.pathmorphing,backgrounds,fit,shapes}
\tikzset{>=latex}

\newtheorem{theorem}{Theorem}
\newtheorem{lemma}{Lemma}
\newtheorem{proposition}{Proposition}
\newtheorem{corollary}{Corollary}
\newtheorem{fact}{Fact}

\newcommand{\remove}[1]{}
\newenvironment{remark}[1][Remark]{\begin{trivlist}
\item[\hskip \labelsep {\bfseries #1}]}{\end{trivlist}}
\newenvironment{definition}[1][Definition.]{\begin{trivlist}
\item[\hskip \labelsep {\bfseries #1}]}{\end{trivlist}}

\newcommand{\Opt}{\mathrm{Opt}}
\newcommand{\real}{\mathbb{R}}

\newcommand{\ints}{\mathbb{Z}}

\newcommand{\eqdef}{~\stackrel{\mbox{\tiny \textnormal{def}}}{=}~}

\newcommand{\set}[1]{\left\{#1\right\}}

\newcommand{\tup}[1]{\left\langle #1 \right\rangle}
\newcommand{\ceil}[1]{\left\lceil #1 \right\rceil}
\newcommand{\floor}[1]{\left\lfloor #1 \right\rfloor}

\newcommand{\tpath}{\textsc{Path}}
\newcommand{\tree}{\textsc{SubTree}}

\newcommand{\fix}[1]{\textcolor{black}{#1}}
\newcommand{\red}[1]{\textcolor{red}{#1}}

\newcommand{\INR}{\mathbf{INR}} 
\newcommand{\EA}{\mathbf{EA}} 
\newcommand{\WI}{\mathbf{WIRE}} 
\newcommand{\Acc}{\mathbf{Acc}} 
\newcommand{\ARN}{\mathbf{ARN}} 
\newcommand{\ARS}{\mathbf{ARS}} 

\newcommand{\IDA}{\mathrm{IDA}^*} 
\newcommand{\Astar}{\mathrm{A}^*} 

\title{Inconsistency and Accuracy of Heuristics with A* Search}

\author{
Hang Dinh\\
Indiana University South Bend\\
htdinh@iusb.edu
\And
Hieu Dinh\\
MathWorks\\
hieu.dinh@mathworks.com
}

\date{\today}               

\begin{document}
\nocopyright
\maketitle              
\thispagestyle{empty}

\begin{abstract}
Many studies in heuristic search suggest that the accuracy of the heuristic used has a positive impact on improving the performance of the search. In another direction, historical research perceives that the performance of heuristic search algorithms, such as A* and IDA*, can be improved by requiring the heuristics to be consistent -- a property satisfied by any perfect heuristic.  
However, a few recent studies show that inconsistent heuristics can also be used to achieve a large improvement in these heuristic search algorithms. These results leave us a natural question: \emph{which property of heuristics, accuracy or consistency/inconsistency, should we focus on when building heuristics}? While there are studies on the heuristic accuracy with the assumption of consistency, no studies on both the inconsistency and the accuracy of heuristics are known to our knowledge.  

In this study, we investigate the relationship between the inconsistency and the accuracy of heuristics with A* search. Our analytical result reveals a correlation between these two properties. We then run experiments on the domain for the Knapsack problem with a family of practical heuristics. Our empirical results show that in many cases, the more accurate heuristics also have higher level of inconsistency and result in fewer node expansions by A*.  
\end{abstract}

\section{Introduction}
Heuristic search has been playing a practical role in solving hard problems. One of the most popular heuristic algorithms is $\Astar$ search \cite{Ref_Hart68formal}, which is essentially best-first search with an additive evaluation $f(x)=g(x)+h(x)$, where $g(x)$ is the cost of the current path from the start node to node $x$, and $h(x)$ is an estimation of the cheapest cost $h^*(x)$ from $x$ to a solution node. The function $h$ is called a \emph{heuristic function}, or \emph{heuristic} for short.
An important property of $\Astar$ search is its admissibility: $\Astar$ will always return an optimal solution if the heuristic $h$ it uses is \emph{admissible}, meaning $h(x)$ never exceeds $h^*(x)$. 

Research on $\Astar$  and other similar heuristic search algorithms, such as $\IDA$ \cite{Ref_Korf85Iterative}, has focused on understanding the impact of properties of the heuristic function on the quality of the search. 
A well-studied subclass of admissible heuristics is the one with the \emph{consistency} property. Heuristic $h$ is called \emph{consistent} if $h(x)\leq c^*(x,x')+h(x')$ for all pairs of nodes $(x,x')$, where $c^*(x,x')$ is the cheapest cost from $x$ to $x'$.  
Consistency was introduced in the original $\Astar$ paper \cite{Ref_Hart68formal} and later became a desirable property of admissible heuristics for two perceptions. First, since the perfect heuristic  $h^*$ is  consistent, it is expected that a good heuristic should also be consistent. The consistency is believed to enable $\Astar$ to forgo reopening nodes \cite[p. 82]{Ref_Pearl84heuristics} and thus can reduce the number of node expansions.  Second, inconsistent admissible heuristics seem rare. In fact, it is assumed by many researchers \cite{Ref_Richard00recent} that ``almost all admissible heuristics are consistent.''  

The portrait of inconsistent heuristics was usually painted negatively until recently, when \citeauthor{Ref_Zahavi07inconsistent}~\shortcite{Ref_Zahavi07inconsistent} discovered that inconsistency is actually not that bad. They  demonstrated by empirical results that in many cases, inconsistency can be used to achieve large performance improvements of $\IDA$. They then promoted the use of inconsistent  heuristics and showed how to turn a consistent heuristic into an inconsistent heuristic using the \emph{bidirectional pathmax} (BPMX) method of \citeauthor{Ref_Felner05Dual}~\shortcite{Ref_Felner05Dual}.  Follow-up studies \cite{Ref_Felner11Inconsistent,Ref_Zhang09search} have also provided positive results of inconsistent heuristics with $\Astar$ search and encouraged researchers to explore inconsistency as a means to further improve the performance of $\Astar$.

In another line of research on heuristics, there have been extensive investigations on the impact of the accuracy of the heuristic on the performance of $\Astar$ (and $\IDA$). While there are a few negative results \cite{Ref_Richard98complexity,Ref_Richard01depth,Ref_Helmert08how}, most studies \cite{Ref_Pohl77practical,Ref_Gaschnig79thesis,Ref_Huyn80probabilistic,Ref_Sen04average,Ref_Dinh07value,Ref_Dinh12Complexity} in this line support the intuition that in many search spaces, improving the accuracy of the heuristic  can improve the efficiency of $\Astar$. Some of the negative results \cite{Ref_Richard98complexity,Ref_Richard01depth} on the benefit of heuristic accuracy were actually obtained under the assumption that the heuristic is consistent. Other negative results only apply to specific planning domains \cite{Ref_Helmert08how} or contrived search spaces with an overwhelming number of solutions \cite{Ref_Dinh12Complexity}.   

In light of the newly discovered benefit of inconsistent heuristics and the well-established positive results on the accuracy of heuristics, it is natural to ask \emph{so which property, consistency/inconsistency or accuracy, of  heuristics really matter to the performance of $\Astar$?. Is there any relationship between these properties of heuristics?} The goal of paper is to address these questions. 


In this work, we first analyze a correlation between inconsistency and accuracy of heuristics. Our analytical result reveals that the level of inconsistency of a heuristic can serve as an upper bound on the level of accuracy of the heuristic (see Theorem \ref{Thm:WIRE-ARN} for details.) We then investigate the relationship between the inconsistency and accuracy of heuristics as well as their impact on the performance of $\Astar$, by running experiments on a practical domain for the Knapsack problem taken from \cite{Ref_Dinh12Complexity}. 

Our study differs from the previous works \cite{Ref_Felner11Inconsistent,Ref_Zhang09search} on inconsistent heuristics with $\Astar$ in both the search space used and the construction of heuristics. 
While \citeauthor{Ref_Felner11Inconsistent} and \citeauthor{Ref_Zhang09search}  use undirected graphs and focus on the reduction in node re-expansions as a benefit of inconsistency, our experiments are done on a directed acyclic graph on which $\Astar$ will never reopen nodes, regardless of the heuristic used. For this search graph, we use a family of heuristics that arise in practice, which allow us to compare the inconsistency level and the accuracy level of many heuristics within this family. Recall that \citeauthor{Ref_Felner11Inconsistent} and \citeauthor{Ref_Zhang09search} incorporated BPMX into $\Astar$ and compared the performance of $\Astar$ with other less well-known heuristic algorithms (B, B', C). However, as pointed out by \citeauthor{Ref_Zahavi07inconsistent}~\shortcite{Ref_Zahavi07inconsistent}, BPMX is only applicable for undirected graphs, thus is inapplicable for the search space we consider.



\section{Preliminaries}\label{Sec_background}
Firstly, we would like to review basic background on $\Astar$ search and introduce our notation.

A typical search problem for $\Astar$ is defined by an edge-weighted search graph $G$ with a start node and a set of goal nodes called solutions. For each graph $G$, we will use $V(G)$ and $E(G)$ to denote the set of vertices and the set of edges of $G$.
We will denote a general search space for $\Astar$ as $(G, c, x_0, S)$, where $G$ is a directed graph, $c: E(G)\to \real^+$ is a function assigning a positive cost to each edge, $x_0\in V(G)$ is the start node, and $S\subset V(G)$ is the set of solution nodes. When $x_0$ and $S$ are not important in the current context, we may only write $(G,c)$. Given a search space $(G,c)$, for each node $x\in V(G)$, 
let $h^*(x)$ denote the cost of a cheapest path from $x$ to a solution node.    

A heuristic function on a search space $(G,c)$ is a function $h: V(G)\to \real^+$, where $h(v)$ is an estimation of  $h^*(x)$, for each $x\in V(G)$.  Since $h^*(s)=0$ for every solution node $s$, we will  assume that a heuristic function must have value zero at every solution node. We will write $\Astar(h)$ to refer to the $\Astar$ search using heuristic $h$. Recall that $\Astar(h)$ is a specialized best-first search algorithm with the evaluation function $f(x)=g(x)+h(x)$, where $g(x)$ is the cost of the \emph{current} path from the start node to node $x$. Details of  the $\Astar(h)$ search on search space $(G,c,x_0, S)$ are described in Algorithm \ref{Alg:Astar_on_graph}. The efficiency of $\Astar$ is usually measured by the number of node expansions, i.e., the executions of Step \ref{Graph_Step_expand} in Algorithm \ref{Alg:Astar_on_graph}.

\begin{algorithm}\caption{$\Astar$ search on search space $(G,c,x_0, S)$ using heuristic $h$ \cite[p. 64]{Ref_Pearl84heuristics}}\label{Alg:Astar_on_graph}
\begin{enumerate}
\item Initialize $\textsc{Open} := \{x_0\}$ and $g(x_0):=0$. 
\item Repeat until \textsc{Open} is empty.
\begin{enumerate}
\item\label{Graph_Step_remove} Remove from $\textsc{Open}$ and place on $\textsc{Closed}$ a node $x$ for which {the function $f=g+h$} is minimum.
\item If $x$ is a solution, i.e., $x\in S$, exit with success and return $x$.
\item\label{Graph_Step_expand} Otherwise, expand $x$, {generating} all its successors.
For each successor $x'$ of $x$,
\begin{enumerate}
 \item If $x'$ is not on $\textsc{Open}$ or $\textsc{Closed}$, estimate $h(x')$ and 
 calculate $f(x')= g(x')+h(x')$ where $g(x')=g(x)+c(x,x')$, and put $x'$ to $\textsc{Open}$  with pointer back to $x$.
 \item\label{Graph_Step_expand2} If $x'$ is on $\textsc{Open}$ or $\textsc{Closed}$, compare $g(x')$ and $g(x)+c(x,x')$.
 If $g(x')> g(x)+c(x,x')$, 
 direct the pointer of $x'$ back to $x$ and reopen $x'$ if it is in $\textsc{Closed}$.
\end{enumerate}
\end{enumerate}
\item Exit with failure.
\end{enumerate}
\end{algorithm}

\paragraph{Informedness and dominance.}
Admissible heuristics also possess a natural \emph{dominance} property \cite[Thm. 7, p. 81]{Ref_Pearl84heuristics}:
for any admissible heuristic functions $h_1$ and $h_2$ on $\mathcal{T}$, if $h_1$ is \emph{more informed than} $h_2$, i.e., $h_1(x) > h_2(x)$ for all non-solution node $x$,  then $\Astar(h_1)$ \emph{dominates} $\Astar(h_2)$, i.e., every node expanded by $\Astar(h_1)$ is also expanded by $\Astar(h_2)$.

\paragraph{Consistency and monotonicity.} The consistency is in fact equivalent to the \emph{monotonicity} \cite[Thm. 8, p. 83]{Ref_Pearl84heuristics}. Precisely, heuristic $h$ on a search space $(G,c)$ is  \emph{consistent} if and only if $$h(x)\leq c(x,x')+h(x')$$ for all edges $(x,x')\in E(G)$. 

\section{Inconsistency and Accuracy}

We now analyze the relationship between  inconsistency and accuracy of heuristics. 
We begin with introducing metrics characterizing the inconsistency of a heuristic.

To characterize the level of inconsistency of a heuristic $h$, \citeauthor{Ref_Zahavi07inconsistent}~\shortcite{Ref_Zahavi07inconsistent} defined the following two terms:
\begin{itemize}
\item \emph{Inconsistency rate of an edge} (IRE): For each edge $e=(u,v)$, let $IRE(h,e)=|h(u)-h(v)|$. The IRE of $h$ is the average $IRE(h,e)$ over all edges $e$ of the search space. 
\item \emph{Inconsistency rate of a node} (IRN): For each node $v$, let $IRN(h,v)$ be the maximal value of $|h(u)-h(v)|$ for any node $u$ adjacent to $v$. The IRN of $h$ is the average $IRN(h,v)$ over all nodes $v$ of the search space.
\end{itemize}

Note that neither IRN nor IRE  defined above takes into account the edge costs. 
If the search space has uniform edge cost, we can say that a consistent heuristic has IRN or IRE at most 1. But if the search space has nonuniform edge costs, we are unable to determine if a heuristic is consistent by just looking up its IRN or IRE.  
Additionally, the metrics IRN and IRE of \citeauthor{Ref_Zahavi07inconsistent}~\shortcite{Ref_Zahavi07inconsistent} were defined for undirected graphs, which are not suitable for the case of search graphs considered in this paper.  Therefore, we define other metrics for the inconsistency to overcome these shortcomings. 
\begin{definition}
Let $h$ be a heuristic on a search space $(G,c)$. 
The \emph{weighted inconsistency rate} of $h$ at edge $e=(x,x')$ is 
\[
\WI(h,e) \eqdef \frac{h(x)-h(x')}{c(x,x')}\,.
\] 
The \emph{weighted inconsistency rate} of $h$, denoted $\WI(h)$, is the average of $\WI(h,e)$ over all edges $e=(x,x')\in E(G)$ where $x$ is a non-solution node.
\end{definition}

The notion of $\WI$ can be seen as a weighted analog of IRE with two minor caveats. First, we use  $(h(x)-h(x'))$ instead of the absolute value $|h(x)-h(x')|$, since the graphs we consider are directed. Second, when computing $\WI(h)$, we do not count $\WI(h,e)$ for edges $e$ from a solution node, because there will be no node expansion made from a solution node. More precisely, if $e=(x,x')$ and $x$ is a solution, then $\WI(h,e)$ has no impact on the search quality.

Clearly, if $h$ is consistent, then  $\WI(h)\leq 1$. The converse, however, is not necessarily true. 
Thus, we define the following metric that can be used to determine if a heuristic is consistent or inconsistent. 

\begin{definition} Let $h$ be a heuristic on a search  space $(G,c)$.
We say that $h$ is \emph{inconsistent at node $x$} if $h(x)> c(x,x')+h(x')$ for some direct successor $x'$ of $x$, i.e., $(x,x')\in E(G)$. The \emph{inconsistent node rate} of $h$, denoted $\INR(h)$, is the ratio of the number of non-solution nodes at which $h$ is inconsistent over the number of all non-solution nodes. 
\end{definition}

In other words, $\INR(h)$ is the probability that $h$ is inconsistent at a random non-solution node. Intuitively, the larger $\INR(h)$, the more inconsistent the heuristic $h$ is. Note that since the heuristic value of any solution node is zero, a heuristic is never inconsistent at a solution node. Hence, we have the following fact:
\begin{fact}
Let $h$ be any heuristic. Then $h$ is inconsistent if and only if $\INR(h)~>~0$. 
\end{fact}

For the accuracy metrics of heuristics, we will adopt the  accuracy notion that measures the distance between the heuristic value and the actual value by a multiplicative factor, which has also been adopted in many previous works \cite{Ref_Gaschnig79thesis,Ref_Huyn80probabilistic,Ref_Sen04average,Ref_Dinh07value,Ref_Dinh12Complexity}

\begin{definition}
Let $h$ be a heuristic function on a search space $(G,c, x_0, S)$. For any non-solution node $x$, we define the \emph{accuracy rate of $h$ at node $x$} to be 
\[
\ARN(h,x)\eqdef \frac{h(x)}{h^*(x)}\,.
\] 
The \emph{accuracy rate of $h$}, denoted $\ARN(h)$, is the average of $\ARN(h,x)$ for all non-solution nodes $x\in V(G)\setminus S$. 
The accuracy rate of $h$ at the start node $x_0$ will be denoted $\ARS(h)$. That is,
\[
\ARS(h) \eqdef \frac{h(x_0)}{h^*(x_0)}\,.
\]
\remove{
The \emph{accuracy rate} of $h$, denoted $\Acc(h)$, is the minimal accuracy rate of $h$ at any non-solution node. That is,
\[
\Acc(h) \eqdef \min_{v\in V(G)\setminus S} \frac{h(v)}{h^*(v)}\,.
\]
}
\end{definition}
This notion of accuracy rate is particularly meaningful for admissible heuristics.
Intuitively, if $h$ is admissible, then the larger $\ARN(h,x)$, the more accurate the heuristic is at node $x$. The accuracy rate is in fact related to the informedness of admissible heuristics: for any two admissible heuristics $h_1$ and $h_2$ on the same search space, $h_1$ is more informed than $h_2$ iff  $\ARN(h_1,x)> \ARN(h_2,x)$ for all non-solution node $x$. 

\remove{
\begin{definition}[Effective accuracy.]The \emph{effective accuracy} of $h$, denoted $\EA(h)$, is the minimal accuracy of $h$ at any non-solution node that is opened by the $\Astar$ using $h$. That is,
\[
\EA(h) \eqdef \min\set{\frac{h(x)}{h^*(x)}\mid \text{$x$ is a non-solution node opened by $A^*(h)$}}\,.
\]
\end{definition}
}

We will now prove a basic relationship between weighted inconsistency rate and accuracy rate. 
\begin{theorem}\label{Thm:WIRE-ARN}
Let $h$ be a heuristic on a search space $(G,c)$ and $x\in V(G)$. If $\WI(h,e)\leq \omega$ for all edges $e$ along a cheapest path from $x$ to a solution node, then $\ARN(h,x)\leq \omega$.
\end{theorem}
\begin{proof}
Let $(x_1,\ldots, x_{\ell})$ be a cheapest path from $x$ to a solution node, where $x_1=x$ and $x_{\ell}$ is a solution, and assume $\WI(h,e)\leq \omega$ for all edges along this path. Then
\[
h(x_i)-h(x_{i+1}) \leq \omega\cdot c(x_{i},x_{i+1}) \quad\forall i=1,\ldots, \ell-1\,.
\]
On the other hand, $h^*(x)=\sum_{i=1}^{\ell-1} c(x_{i},x_{i+1})$. It follows that
\[
\begin{split}
h(x_1)-h(x_{\ell}) 
&= \sum_{i=1}^{\ell-1} (h(x_i)-h(x_{i+1})) \\
&\leq \sum_{i=1}^{\ell-1} \omega\cdot c(x_{i},x_{i+1}) = \omega h^*(x)\,.
\end{split}
\]
Since $x_{\ell}$ is a solution, $h(x_{\ell})=0$ by assumption. Thus, we have $h(x)=h(x)-h(x_{\ell})\leq \omega h^*(x)$.
\end{proof}

\begin{corollary}
For any heuristic $h$, if $\WI(h,e)\leq \omega$ for all edges $e$ then $\ARN(h,x)\leq \omega $ for all nodes $x$. 
\end{corollary}

This means that an upper bound on the weighted inconsistency rates of a heuristic $h$ is also an upper bound on the accuracy rates of $h$. In particular, if the heuristic $h$ is consistent, then the less $\WI(h)$, the less accurate $h$ can be. 
This suggests that imposing consistency on the heuristic can prevent improving the heuristic accuracy.

\section{Experiments with Knapsack Problem}
We will experimentally investigate the relationship between inconsistency and accuracy of heuristics on a practical domain namely the Knapsack problem. This problem is  NP-complete and has applications in many fields, from business to cryptography. Our heuristics will also be built in a practical way, based on an approximation algorithm for the Knapsack problem.
 
\subsection{Search Model for Knapsack}

A Knapsack instance is denoted by a tuple $\tup{X, p, w, C}$, where $X$ is a finite set of items, $p: X\to \ints^+$ is a function assigning profit to each item, $w: X\to \ints^+$ is a function assigning weight to each item, and $C>0$ is the capacity of the knapsack. Recall that the knapsack problem is to find a subset $X^* \subseteq X$ of items whose total weight does not exceed capacity $C$ and whose total profit is maximal. We will write $p(X)$ and $w(X)$ to denote the total profit and the total weight, respectively, of all items in $X$, i.e.,  $w(X)=\sum_{i\in X}w(i)$ and $p(X)=\sum_{i\in X}p (i)$. For each positive integer $n$, let $[n]=\set{1,2,\ldots, n}$, and  we may simply write $[n]$ to represent a set of $n$ items. 

Here we will adopt the search model for the Knapsack problem that has been employed in \cite{Ref_Dinh12Complexity}.
In particular, consider the Knapsack instance $\tup{[n],p,w,C}$. The search graph for this instance is a directed graph, in which each node (or state) is a nonempty subset $X\subseteq[n]$ and each edge $(X, X')$ corresponds to the removal of an item $i\in X$ so that $X\setminus \set{i}=X'$. The cost of such an edge $(X,X')$ is the profit of the removed item $i$. See Figure~\ref{Fig:Knapsack-search-space-4} for an example of edges from a node $X=\set{1,2,3,4}$. The start node is the set $[n]$. A node $X$ is designated as a solution if $w(X)\leq C$.

An important property of this search space is that every path from node $X$ to node $X'$ has the same total cost, which equals the total profit of items in $X\setminus X'$. Thanks to this property, $\Astar$ will avoid reopening nodes from CLOSED. Thus, consistent heuristics are not needed in this case.

\begin{figure}[ht]
\centering
    \begin{tikzpicture}[scale=1,state/.style={ellipse,draw=black!70,thick,
        inner sep=0pt,minimum height=4mm,minimum width=2mm}]
      \node[state] (1234) [fill=black!20] at (0,0) {$\{1,2,3,4\}$};
      \node[state] (123) at (-3,-1.75) {$\{1,2,3\}$};
      \node[state] (124) at (-1,-1.75) {$\{1,2,4\}$};
      \node[state] (134) at (1,-1.75) {$\{1,3,4\}$};
      \node[state] (234) at (3.5,-1.75) {$\{2,3,4\}$};
      \path[->,thick] 
      (1234) edge node [above]{$p(4)$} (123)      
      (1234) edge node [right]{$p(3)$} (124)      
      (1234) edge node [right]{$p(2)$} (134)      
      (1234) edge node [above]{$p(1)$} (234);
    \end{tikzpicture}  
\caption{
Edges from a node in a Knapsack search space.
 }
\label{Fig:Knapsack-search-space-4}
\end{figure}
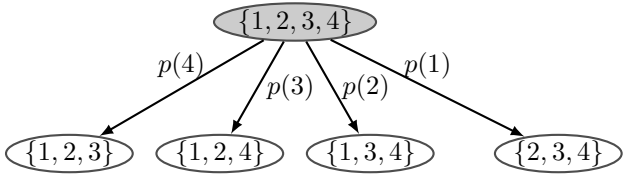

\subsection{Heuristic Construction}
Consider the search space for a Knapsack instance $\tup{[n], p,w, C}$.
We construct efficient admissible heuristics on this search space in a similar way to the construction of \citeauthor{Ref_Dinh12Complexity}~\shortcite{Ref_Dinh12Complexity}, but without constraints to obtain an accuracy guarantee, 
which is a lower bound on the minimal accurate rate.
The main ingredient of this construction is an FPTAS (\emph{Fully Polynomial Time Approximation Scheme}) due to \citeauthor{Ref_Ibarra75Fast}~\shortcite{Ref_Ibarra75Fast}, which is described in Algorithm \ref{Alg:FPTAS-Knapsack} below. This FPTAS is an algorithm, denoted $\mathcal{A}$, that returns a solution with total profit at least $(1-\epsilon)\Opt(X)$ to each Knapsack instance $\tup{X,p,w,c}$ and runs in time 
$O\left(|X|^3/\epsilon\right)$ \cite[p. 70]{Ref_Vazirani01approximation}, for any given $\epsilon\in(0,1)$, where  $\Opt(X)$ is the total profit of an optimal solution to the Knapsack instance $\tup{X,p,w,c}$.
For each subset $X\subseteq [n]$,  
let $\mathcal{A}_{\epsilon}(X)$ denote the total profit of the solution returned by algorithm $\mathcal{A}$ with error parameter $\epsilon$  to the Knapsack instance  $\tup{X, p,w, c}$.
Then for any $\epsilon\in (0,1)$,
\[
(1-\epsilon)\Opt(X)\leq \mathcal{A}_{\epsilon}(X) \leq \Opt(X)\,.
\]
Since $h^*(X)=p(X)-\Opt(X)$, it follows that
\begin{equation}\label{Eq:Knapsack-appr}
p(X) - \frac{\mathcal{A}_{\epsilon}(X)}{1-\epsilon} \leq h^*(X) \leq p(X) -\mathcal{A}_{\epsilon}(X)\,.
\end{equation}
Note that the lower bound $p(X) - \frac{\mathcal{A}_{\epsilon}(X)}{1-\epsilon} $ can fall below zero, especially for large $\epsilon$. 
Hence, for each parameter $\epsilon\in (0,1)$, we define the following heuristic $h_{\epsilon}$ whose admissibility is guaranteed: for any non-solution node $X$,
\[
h_{\epsilon}(X) \eqdef \max\set{p(X) - \frac{\mathcal{A}_{\epsilon}(X)}{1-\epsilon}, 0}\,.
\] 
 
Since the running time to compute $\mathcal{A}_{\epsilon}(X)$ is $O\left(|X|^3 \epsilon^{-1}\right)$, the running time to compute $h_{\epsilon}(X)$ is also $O\left(|X|^3 \epsilon^{-1}\right)$,  which is polynomial in both $n$ and $\epsilon^{-1} $. 

\begin{algorithm}\caption{FPTAS for Knapsack \cite[p. 70]{Ref_Vazirani01approximation}}\label{Alg:FPTAS-Knapsack}
\textbf{Given}: Knapsack instance $\tup{X,p,w,C}$, and error parameter $\epsilon\in (0,1)$. Let $X=\set{a_1,\ldots, a_j}$.
\begin{enumerate}
\item Let $P=\max_{i\in X} p(i)$ and $K=\epsilon P/|X|$. 
\item For each item $i\in X$, define new profit $p'(i)=\floor{p(i)/K}$. Let $P'=\floor{P/K}$.
\item Let $S_{i,q}$ denote a subset of $\set{a_1,\ldots,a_i}$ so that $p'(S_{i,q})=q$ and $w(S_{i,q})$ is minimal, and let $w_{i,q}:=w(S_{i,q})$ (if no such a set exists, let $w_{i,q}=\infty$). 
Use dynamic programming to compute $w_{i,q}$ and $S_{i,q}$ for all $i\in\set{1,\ldots,|X|}$ and $q\in\set{0,1,\ldots,|X|P'}$. 
\remove{
as follows:
	\begin{enumerate}
	\item Initialize $w_{i,q}=\infty$ for all $i, q$.
	\item For each $q=0,1\ldots,|X|P'$: If $p'(a_1)=q$, then set $S_{1,q}=\set{a_1}$ and $w_{1,q}=w(a_1)$.
	\item For  $i=2,\ldots,|X|$, for $q=0,1,\ldots, |X|P'$:\\ If $p'(a_i)<= q$ AND $w_{i-1,q}> w(a_i)+w_{i-1, q-p'(a_i)}$, then set $S_{i,q}=S_{i-1, q-p'(a_i)} \cup\set{a_i}$ and $w_{i,q}=w(a_i)+w_{i-1, q-p'(a_i)}$. Otherwise, set $S_{i,q}=S_{i-1,q}$ and $w_{i,q}=w_{i-1,q}$.
	\end{enumerate}
	}
\item Find the most profitable set $S'$ among $S_{i,q}$ with $w_{i,q}\leq C$.
\item Return $S'$.
\end{enumerate}
\end{algorithm}

While there is no accuracy guarantee on $h_{\epsilon}$, it is intuitive to expect the growth in the accuracy of $h_{\epsilon}$ by reducing the FPTAS error parameter $\epsilon$. It then remains to find if the inconsistency of $h_{\epsilon}$ will also grow as $\epsilon$ decreases.

\subsection{Experiments}
For our experiments, we generate hard Knapsack instances $\tup{[n],p,w,C}$ from the following Knapsack instance distributions, or ``types,'' which are identified by \citeauthor{Ref_Pisinger05Where}~\shortcite{Ref_Pisinger05Where} as difficult instances for best-known exact algorithms: 
\begin{description}
\item \textbf{Strongly correlated:}
For each item $i\in [n]$, choose its weight $w(i)$ as a random integer in the range $[1,R]$ and set its profit $p(i)=w(i) +R/10$. This correlation between  weights and profits reflects  real-life situations where the profit of an item is proportional to its weight plus some fixed charge.
\item\textbf{Inverse strongly correlated:}
For each item $i\in [n]$, choose its profit $p(i)$ as a random integer in the range $[1,R]$ and set its weight $w(i)=p(i) +R/10$.  
\item\textbf{Almost strongly correlated:}
For each item $i\in [n]$, choose its weight $w(i)$ as a random integer in the range $[1,R]$ and choose its profit $p(i)$ as a random integer in the range $[w(i) +R/10-R/500, w(i) +R/10+R/500]$.  
\item\textbf{Subset sum:}
For each item $i\in [n]$, choose its weight $w(i)$ as a random integer in the range $[1,R]$ and set its profit $p(i)=w(i)$. Knapsack instances of this type are instances of the subset sum problem.

\item\textbf{Uncorrelated with similar weight:}
For each item $i\in [n]$, choose its weight $w(i)$ as a random integer in the range $[100000,100100]$ and choose its profit $p(i)$ as a random integer in $[1,R]$.  

\item\textbf{Multiple strongly correlated:}
For each item $i\in [n]$, choose its weight $w(i)$ as a random integer in the range $[1,R]$. If $w(i)$ is divisible by $6$, set the profit $p(i)=w(i) +3R/10$. Otherwise, set $p(i)=w(i) +2R/10$. This family of instances is denoted $\mathsf{mstr}(3R/10, 2R/10, 6)$ by \citeauthor{Ref_Pisinger05Where}~\shortcite{Ref_Pisinger05Where} and is the most difficult family of ``multiple strongly correlated instances'' considered by \citeauthor{Ref_Pisinger05Where}. 

\item\textbf{Profit ceiling:} For each item $i\in [n]$, choose its weight $w(i)$ as a random integer in the range $[1,R]$ and set its profit $p(i)=3\ceil{w(i)/3}$. This family of instances is denoted $\mathsf{pceil}(3)$, which resulted in sufficiently difficult instances for experiments of \citeauthor{Ref_Pisinger05Where}~\shortcite{Ref_Pisinger05Where}. 

\end{description}
Here we set {the data range parameter} $R := 1000$. The knapsack capacity is chosen as  $C=({t}/101)w([n])$, where ${t}$ is a random integer in the range $[30,70]$. 

In our experiments, we generate one Knapsack instance $\tup{[n],p,w,C}$ of each type above.  
For each Knapsack instance  generated, we run a series of $\Astar(h_{_{\epsilon}})$  with different values of $\epsilon$, as well as breath-first search. We chose the sample points for $\epsilon$ with two consecutive points differed by a factor of $2$, so as to clearly see the change in the number of node expansions made by $\Astar(h_{_{\epsilon}})$.  

The main challenge of these experiments is to compute $\ARN(h_{\epsilon})$. 
It is typically too expensive to compute $\ARN$ of a heuristic on a practical search space, because it requires computing  $h^*(x)$ exactly for all non-solution nodes $x$. For the Knapsack search space, we can also rely on the given FPTAS $\mathcal{A}$ to compute $h^*(X)$ for each node $X\subseteq[n]$. 
Our computation is based on the following proposition:

\begin{proposition}
For any $0<\gamma<1/\Opt(X)$,
\begin{equation}\label{Eq:h*(X)}
h^*(X) = \floor{p(X) -\mathcal{A}_{\gamma}(X)}\,.
\end{equation}
\end{proposition}
\begin{proof}
Since $\mathcal{A}_{\gamma}(X) \geq (1-\gamma)\Opt(X)$, we have
\[
\begin{split}
p(X) -\mathcal{A}_{\gamma}(X) &\leq p(X)-(1-\gamma)\Opt(X)\\
& = h^*(X) +\gamma \Opt(X) < h^*(X)+1\,.
\end{split}
\]
On the other hand, from Equation \eqref{Eq:Knapsack-appr}, we have $h^*(X) \leq p(X) -\mathcal{A}_{\gamma}(X)$. The proof is completed by noting that  $h^*(X)$ is an integer\,.
\end{proof}

Since $\Opt(X)< \min\set{p(X), \Opt([n])+1}$ for all non-solution node $X\subseteq [n]$,
we compute $h^*(X)$ as in Equation \eqref{Eq:h*(X)} with
\begin{equation}\label{Eq:gamma}
\gamma={1}/{\min\set{p(X), \Opt([n])+1}}\,.
\end{equation}
The value of $\Opt([n])$ is obtained after running $\Astar(h_{\epsilon})$, which returns the optimal solution cost $h^*([n])=p([n])-~\Opt([n])$. 

While using the FPTAS could save us a considerable amount of time, computing $\mathcal{A}_{\gamma}(X)$ with $\gamma$ specified in \eqref{Eq:gamma} is still time-comsuming -- it actually has pseudo-polynomial time complexity. 
As such, we limit our experiments to Knapsack instances of relatively small size ($n=20$), for which each $\ARN(h_{\epsilon})$ can be computed within 10 hours.

Detailed results of our experiments are shown in Tables \ref{Table:Strc}--
\ref{Table:ProfitCeil}, each table corresponds to a Knapsack instance type listed above. In each of these tables, the first column gives the values of the FPTAS error parameter $\epsilon$. The row  with ``BFS'' in the first column presents the breath-first search. The column ``Node Exps'' contains the number of node expansions made by each search. The last four columns show data for $\ARS(h_{\epsilon})$, $\ARN(h_{\epsilon})$, $\INR(h_{\epsilon})$, and $\WI(h_{\epsilon})$, respectively. Data of the Multiple Strongly Correlated type (in Table \ref{Table:MulStrc}) are not available for all  sample values of $\epsilon$ due to lack of time. 
Figure \ref{Fig:KnapsackAverage} shows the trend of $\ARS(h_{\epsilon})$, $\ARN(h_{\epsilon})$, $\INR(h_{\epsilon})$ and $\WI(h_{\epsilon})$,  averaged over all Knapsack instances, but the one of the Multiple Strongly Correlated type, in our experiments. Recall that all these Knapsack instances  have the same number of items, $n=20$, thus have the same search graph. 

Our data show that when the accuracy metrics $\ARN(h_{\epsilon})$ and $\ARS(h_{\epsilon})$ grow, then so do the inconsistency metrics $\INR(h_{\epsilon})$ and $\WI(h_{\epsilon})$. 
Loosely speaking, the accuracy and the inconsistency level of heuristics $h_{\epsilon}$ are somewhat correlated. This could explain why the inconsistent admissible heuristics can improve the efficiency of $\Astar$.
We observe, in addition, that for small $\epsilon$ ($<0.02$), the values of $\ARN$ and $\INR$ are close to each other in many instances, such as Strongly Correlated, Almost Strongly Correlated, Subset Sum, and Multiple Strongly Correlated. 
Regarding the performance of $\Astar$, our data also show a significant reduction in the number of node expansions when the heuristic is more accurate, and thus more inconsistent.

\begin{figure}[!t]
\centering
 \begin{tikzpicture}
   \begin{axis}[
     ylabel={},xlabel={FPTAS error parameter $\epsilon$}, 
     legend style={legend pos=north east, legend cell align=left}]
     \addplot[color=blue,mark=x,domain=0.0:0.5,thick,dashed]
                    coordinates {
(0.0016,0.998208)
(0.0032,0.996429333)
(0.0064,0.9928155)
(0.0128,0.985491333)
(0.0256,0.970511833)
(0.0512,0.939629167)
(0.1024,0.871867167)
(0.2048,0.7110395)
(0.4096,0.314347667)
	};
     \addlegendentry{ARS}

     \addplot[color=blue,mark=*,domain=0.0:0.5,thick]
                    coordinates {
(0.0016,0.976548333)
(0.0032,0.957718)
(0.0064,0.919673333)
(0.0128,0.856378167)
(0.0256,0.75934)
(0.0512,0.611018333)
(0.1024,0.407965667)
(0.2048,0.183286667)
(0.4096,0.021184167)
	};
     \addlegendentry{ARN}

     \addplot[color=red,mark=square*,domain=0.0:.5,thick]
                    coordinates {
(0.0016,0.904193167)
(0.0032,0.902916833)
(0.0064,0.897835167)
(0.0128,0.883174)
(0.0256,0.848439)
(0.0512,0.778592)
(0.1024,0.627149833)
(0.2048,0.376294167)
(0.4096,0.074935833)
	};
     \addlegendentry{IRN}

     \addplot[color=red,mark=diamond*,domain=0.0:0.5,thick]
                    coordinates {
(0.0016,3.1050965)
(0.0032,3.084059333)
(0.0064,3.0422405)
(0.0128,2.960107)
(0.0256,2.802683167)
(0.0512,2.498024333)
(0.1024,1.947922167)
(0.2048,1.089671833)
(0.4096,0.186747)	};
     \addlegendentry{WIRE}
     \end{axis}
 \end{tikzpicture}
\caption{Results averaged over all Knapsack instances (of the same size $n=20$) from Tables \ref{Table:Strc}, \ref{Table:InStrc}, \ref{Table:AlStrc}, \ref{Table:SubsetSum}, \ref{Table:UnSimWeight}, 
and \ref{Table:ProfitCeil}}\label{Fig:KnapsackAverage}
\vspace{-3cm}
\end{figure}
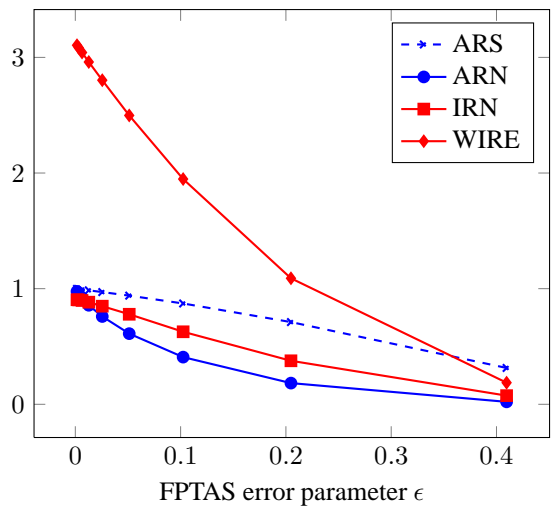

\begin{table}[!h]
  \centering 
\scriptsize{
  \begin{tabular}{|l|r|l|l|l|r|}
\hline
$\epsilon$ &  Node Exps &  $\ARS$ &  $\ARN$ &  $\INR$ &  $\WI$\\
\hline
0.0016 &  332  & 0.9986 & 0.9918 & 0.9362 & 2.9378\\ 
0.0032 &  532  & 0.9973 & 0.9834 & 0.9363 & 2.9238\\ 
0.0064 &  627  & 0.9945 & 0.9667 & 0.9355 & 2.8956\\ 
0.0128 &  767  & 0.9889 & 0.9328 & 0.9328 & 2.8381\\ 
0.0256 &  9,921  & 0.9773 & 0.8644 & 0.9211 & 2.7249\\ 
0.0512 &  73,212  & 0.9538 & 0.7292 & 0.8709 & 2.4891\\ 
0.1024 &  446,925  & 0.9023 & 0.5104 & 0.7501 & 2.0087\\ 
0.2048 &  609,517  & 0.7782 & 0.2158 & 0.4711 & 1.1008\\ 
0.4096 &  609,535  & 0.4159 & 0.0054 & 0.0294 & 0.0519\\
\hline 
BFS &  609,501  &  &  &  & \\
\hline
\end{tabular}
}\vspace{-0.2cm}
  \caption{Results of Knapsack instance type \textbf{Strongly Correlated}}\label{Table:Strc}
\end{table}

\begin{table}[!h]
  \centering 
\scriptsize{
  \begin{tabular}{|l|r|l|l|l|r|}
\hline
$\epsilon$ &  Node Exps &  $\ARS$ &  $\ARN$ &  $\INR$ &  $\WI$\\
\hline
0.0016 &  4,095  & 0.9975 & 0.9651 & 0.7857 & 1.6031\\ 
0.0032 &  14,831  & 0.9950 & 0.9303 & 0.7759 & 1.5802\\ 
0.0064 &  19,576  & 0.9898 & 0.8586 & 0.7767 & 1.5349\\ 
0.0128 &  33,852  & 0.9794 & 0.7427 & 0.7552 & 1.4490\\ 
0.0256 &  99,985  & 0.9587 & 0.5826 & 0.6893 & 1.2883\\ 
0.0512 &  179,001  & 0.9149 & 0.3664 & 0.5455 & 0.9967\\ 
0.1024 &  190,346  & 0.8176 & 0.1467 & 0.3201 & 0.5434\\ 
0.2048 &  190,357  & 0.5886 & 0.0140 & 0.0573 & 0.0859\\ 
0.4096 &  190,352  & 0 & 0 & 0 & 0\\
\hline 
BFS &  190,352  &  &  &  & \\
\hline
\end{tabular}
}\vspace{-0.2cm}
  \caption{Results of Knapsack instance type: \textbf{Inverse Strongly Correlated}}\label{Table:InStrc}
\end{table}

\begin{table}[!h]
  \centering 
\scriptsize{
  \begin{tabular}{|l|r|l|l|l|r|}
\hline
$\epsilon$ &  Node Exps &  $\ARS$ &  $\ARN$ &  $\INR$ &  $\WI$\\
\hline
0.0016 &  511  & 0.9992 & 0.9963 & 0.9782 & 4.7133\\ 
0.0032 &  766  & 0.9983 & 0.9925 & 0.9783 & 4.7035\\ 
0.0064 &  766  & 0.9967 & 0.9850 & 0.9781 & 4.6799\\ 
0.0128 &  766  & 0.9933 & 0.9697 & 0.9769 & 4.6360\\ 
0.0256 &  19,184  & 0.9864 & 0.9386 & 0.9731 & 4.5496\\ 
0.0512 &  116,003  & 0.9730 & 0.8741 & 0.9661 & 4.3672\\ 
0.1024 &  627,996  & 0.9427 & 0.7524 & 0.9373 & 3.9754\\ 
0.2048 &  921,291  & 0.8713 & 0.5116 & 0.8185 & 3.0808\\ 
0.4096 &  921,755  & 0.6518 & 0.1166 & 0.3824 & 0.9986\\
\hline 
BFS &  921,751  &  &  &  & \\
\hline
\end{tabular}
}\vspace{-0.2cm}
  \caption{Results of Knapsack instance type: \textbf{Almost Strongly Correlated}}\label{Table:AlStrc}
\end{table}

\begin{table}[!h]
  \centering 
\scriptsize{
  \begin{tabular}{|l|r|l|l|l|r|}
\hline
$\epsilon$ &  Node Exps &  $\ARS$ &  $\ARN$ &  $\INR$ &  $\WI$\\
\hline
0.0016 &  363,190  & 0.9985 & 0.9845 & 0.9399 & 3.7904\\ 
0.0032 &  427,384  & 0.9971 & 0.9689 & 0.9398 & 3.7649\\ 
0.0064 &  484,978  & 0.9940 & 0.9395 & 0.9352 & 3.7177\\ 
0.0128 &  528,999  & 0.9881 & 0.8832 & 0.9291 & 3.6237\\ 
0.0256 &  552,068  & 0.9757 & 0.7867 & 0.8968 & 3.4356\\ 
0.0512 &  562,801  & 0.9501 & 0.6374 & 0.8412 & 3.0607\\ 
0.1024 &  568,781  & 0.8956 & 0.4078 & 0.6967 & 2.3275\\ 
0.2048 &  569,621  & 0.7620 & 0.1387 & 0.3787 & 1.0798\\ 
0.4096 &  569,621  & 0.3591 & 0.0015 & 0.0138 & 0.0232\\ 
\hline
BFS &  569,621  &  &  &  & \\
\hline
\end{tabular}
}\vspace{-0.2cm}
  \caption{Results of Knapsack instance type: \textbf{Subset Sum}}\label{Table:SubsetSum}
\end{table}

\begin{table}[!h]
  \centering 
\scriptsize{
  \begin{tabular}{|l|r|l|l|l|r|}
\hline
$\epsilon$ &  Node Exps &  $\ARS$ &  $\ARN$ &  $\INR$ &  $\WI$\\
\hline
0.0016 &  24,573  & 0.9987 & 0.9625 & 0.8817 & 3.1518\\ 
0.0032 &  40,951  & 0.9975 & 0.9524 & 0.8895 & 3.1368\\ 
0.0064 &  53,235  & 0.9950 & 0.9321 & 0.8821 & 3.1064\\ 
0.0128 &  65,517  & 0.9899 & 0.8913 & 0.8710 & 3.0449\\ 
0.0256 &  135,083  & 0.9793 & 0.8319 & 0.8598 & 2.9282\\ 
0.0512 &  236,660  & 0.9579 & 0.7233 & 0.8633 & 2.6829\\ 
0.1024 &  448,451  & 0.9116 & 0.5142 & 0.7417 & 2.1757\\ 
0.2048 &  681,705  & 0.7984 & 0.2125 & 0.4915 & 1.1238\\ 
0.4096 &  698,995  & 0.4592 & 0.0035 & 0.0241 & 0.0468\\ 
\hline
BFS &  699,037  &  &  &  & \\
\hline
\end{tabular}
}\vspace{-0.2cm}
  \caption{Results of Knapsack instance type: \textbf{Uncorrelated with Similar Weight}}\label{Table:UnSimWeight}
\end{table}

\begin{table}[!h]
  \centering 
\scriptsize{
  \begin{tabular}{|l|r|l|l|l|r|}
\hline
$\epsilon$ &  Node Exps &  $\ARS$ &  $\ARN$ &  $\INR$ &  $\WI$\\
\hline
0.0016 &  511  & 0.9990 & 0.9964 & 0.9790 & 4.0425\\ 
0.0032 &  511  & 0.9981 & 0.9927 & 0.9790 & 4.0317\\ 
0.0064 &  511  & 0.9961 & 0.9854 & 0.9787 & 4.0120\\ 
0.0128 &  766  & 0.9920 & 0.9706 & 0.9779 & 3.9737\\ 
0.0256 &  1,338  & 0.9837 & 0.9402 & 0.9750 & 3.8915\\ 
0.0512 &  14,429  & 0.9670 & 0.8770 & 0.9621 & 3.7233\\ 
0.1024 &  85,837  & 0.9298 & 0.7458 & 0.9187 & 3.3614\\
\hline
\end{tabular}
}\vspace{-0.2cm}
  \caption{Results of Knapsack instance type: \textbf{Multiple Strongly Correlated}}\label{Table:MulStrc}
\end{table}

\begin{table}[!h]
  \centering 
\scriptsize{
  \begin{tabular}{|l|r|l|l|l|r|}
\hline
$\epsilon$ &  Node Exps &  $\ARS$ &  $\ARN$ &  $\INR$ &  $\WI$\\
\hline
0.0016 &  43,262  & 0.9967 & 0.9592 & 0.9034 & 2.4341\\ 
0.0032 &  69,291  & 0.9935 & 0.9188 & 0.8978 & 2.3952\\ 
0.0064 &  86,397  & 0.9869 & 0.8361 & 0.8795 & 2.3188\\ 
0.0128 &  94,178  & 0.9733 & 0.7187 & 0.8340 & 2.1689\\ 
0.0256 &  97,350  & 0.9457 & 0.5518 & 0.7504 & 1.8894\\ 
0.0512 &  99,497  & 0.8882 & 0.3356 & 0.5846 & 1.3916\\ 
0.1024 &  100,366  & 0.7615 & 0.1163 & 0.3171 & 0.6568\\ 
0.2048 &  100,351  & 0.4677 & 0.0072 & 0.0406 & 0.0669\\ 
0.4096 &  100,364  & 0 & 0 & 0 & 0\\ 
\hline
BFS &  100,364  &  &  &  & \\
\hline
\end{tabular}
}\vspace{-0.2cm}
  \caption{Results of Knapsack instance type: \textbf{Profit Ceiling}}\label{Table:ProfitCeil}
\end{table}

\section{Conclusions and Future Work}
This work provides evidence that the inconsistency and accuracy of heuristics are related. Theoretical evidence suggests that the heuristic accuracy could be upper-bounded by its level of inconsistency. Thus, requiring the heuristic to be consistent could limit the room to improve its accuracy.  Empirical evidence with a family of practical admissible heuristics on Knapsack domains  shows that the more accurate the heuristic, the more inconsistent it is.  The experiments in this work also provide positive results about accurate heuristics and inconsistent admissible heuristics, that is, both the accuracy and the inconsistency of the heuristic can be used to improve the performance of $\Astar$.

Still, further investigation on both the inconsistency and accuracy of heuristics should be carried out. In particular, we have the following goals in mind for our future work:
\begin{enumerate}
\item Investigate the relationship between $\ARN(h)$ and $\INR(h)$ in general cases. 
\item Establish good bounds on the number of node expansions in terms of both accuracy  and inconsistency metrics of the heuristic used. 
\end{enumerate}

\newpage
\bibliographystyle{aaai} 
\bibliography{Astar}

\end{document}